\documentclass{article}

\usepackage{microtype}
\usepackage{graphicx}
\graphicspath{ {resources/} }
\usepackage{subfigure}
\usepackage{float}
\usepackage{booktabs} 
\usepackage{multirow}

\usepackage{amsmath,amsfonts,bm}

\def\eqref#1{equation~\ref{#1}}

\def\1{\bm{1}}

\def\vmu{{\bm{\mu}}}

\def\vg{{\bm{g}}}

\def\vw{{\bm{w}}}
\def\vx{{\bm{x}}}

\DeclareMathAlphabet{\mathsfit}{\encodingdefault}{\sfdefault}{m}{sl}
\SetMathAlphabet{\mathsfit}{bold}{\encodingdefault}{\sfdefault}{bx}{n}

\newcommand{\loss}{\ell}

\newcommand{\minibatch}{\mathcal{M}}

\newcommand{\dkls}[3]{\mathbb{D}_{KL}^{#1}[#2 \, \|\, #3]}

\newcommand\cut[1]{}

\newcommand{\squishlist}{
   \begin{list}{$\bullet$}
    { \setlength{\itemsep}{0pt}      \setlength{\parsep}{3pt}
      \setlength{\topsep}{3pt}       \setlength{\partopsep}{0pt}
      \setlength{\leftmargin}{1.5em} \setlength{\labelwidth}{1em}
      \setlength{\labelsep}{0.5em} } }

\newcommand{\squishlisttwo}{
   \begin{list}{$\bullet$}
    { \setlength{\itemsep}{0pt}    \setlength{\parsep}{0pt}
      \setlength{\topsep}{0pt}     \setlength{\partopsep}{0pt}
      \setlength{\leftmargin}{2em} \setlength{\labelwidth}{1.5em}
      \setlength{\labelsep}{0.5em} } }

\newcommand{\squishend}{
    \end{list}  }

{}
{}
{}
{}
{}

\newcommand{\half}{\mbox{$\frac{1}{2}$}}

\newcommand{\sqr}[1]{\left[#1\right]}

\newcommand{\myexpect}{\mathbb{E}}

\newcommand{\myvecsym}[1]{\mbox{$\boldsymbol{#1}$}}

\newcommand{\vepsilon}{\mbox{$\myvecsym{\epsilon}$}}

\newcommand{\vlambda}{\mbox{$\myvecsym{\lambda}$}}

\newcommand{\calD}{\mbox{${\cal D}$}}

\newcommand{\data}{\calD}

\newcommand{\be}{\begin{equation}}
\newcommand{\ee}{\end{equation}}
\newcommand{\bea}{\begin{eqnarray}}
\newcommand{\eea}{\end{eqnarray}}
\newcommand{\beaa}{\begin{eqnarray*}}
\newcommand{\eeaa}{\end{eqnarray*}}

\usepackage{algorithm}
\usepackage{url}

\usepackage[noend]{algpseudocode}

\usepackage{amsmath}

\newtheorem{lemma}{Lemma}
\newtheorem{proof}{Proof}

\providecommand{\openbox}{\leavevmode
  \hbox to.77778em{%
  \hfil\vrule
  \vbox to.675em{\hrule width.6em\vfil\hrule}%
  \vrule\hfil}}
\makeatletter
\DeclareRobustCommand{\qed}{%
  \ifmmode
    \eqno \def\@badmath{$$}
    \let\eqno\relax \let\leqno\relax \let\veqno\relax
    \hbox{\openbox}%
  \else
    \leavevmode\unskip\penalty9999 \hbox{}\nobreak\hfill
    \quad\hbox{\openbox}%
  \fi
}
\makeatother

\newcommand\scalemath[2]{\scalebox{#1}{\mbox{\ensuremath{\displaystyle #2}}}} 

\usepackage{hyperref}

\usepackage[accepted]{icml2020}

\icmltitlerunning{Training Binary Neural Networks using the Bayesian Learning Rule}

\begin{document}

\twocolumn[
\icmltitle{Training Binary Neural Networks using the Bayesian Learning Rule}



\icmlsetsymbol{equal}{*}

\begin{icmlauthorlist}
\icmlauthor{Xiangming Meng}{to}
\icmlauthor{Roman Bachmann\textsuperscript{*}}{too}
\icmlauthor{Mohammad Emtiyaz Khan}{to}
\end{icmlauthorlist}

\icmlaffiliation{to}{RIKEN Center for Advanced Intelligence Project (AIP), Tokyo, Japan.}
\icmlaffiliation{too}{École polytechnique fédérale de Lausanne (EPFL), Lausanne, Switzerland}

\icmlcorrespondingauthor{Mohammad Emtiyaz Khan}{emtiyaz.khan@riken.jp}

\icmlkeywords{Deep learning, Binary neural networks, Bayesian deep learning}

\vskip 0.3in

]



\printAffiliationsAndNotice{}  

\begin{abstract}
Neural networks with binary weights are computation-efficient and hardware-friendly, but their training is challenging because it involves a discrete optimization problem.
Surprisingly, ignoring the discrete nature of the problem and using gradient-based methods, such as the Straight-Through Estimator, still works well in practice.
This raises the question: are there principled approaches which justify such methods?
In this paper, we propose such an approach using the \emph{Bayesian learning rule}.
The rule, when applied to estimate a Bernoulli distribution over the binary weights, results in an algorithm which justifies some of the algorithmic choices made by the previous approaches.
The algorithm not only obtains state-of-the-art performance, but also enables uncertainty estimation for continual learning to avoid catastrophic forgetting.
Our work provides a principled approach for training binary neural networks which justifies and extends existing approaches. 

\end{abstract}

\section{Introduction}
Deep neural networks (DNNs) have been remarkably successful in machine learning but their training and deployment requires a high energy budget and hinders their application to resource-constrained devices, such as mobile phones, wearables, and IoT devices. 
Binary neural networks (BiNNs), where weights and/or activations are restricted to binary values, are one promising solution to address this issue \citep{courbariaux2016binarized,courbariaux2015binaryconnect}.
Compared to full-precision DNNs, e.g., using 32-bits, using BiNNs directly gives a 32 times reduction in the model size.
Further computational efficiency is obtained by using specialized hardware, e.g., by replacing the multiplication and addition operations with the bit-wise {\tt xnor} and {\tt bitcount} operations \citep{rastegari2016xnor,mishra2017wrpn,bethge2020meliusnet}. 
In the near future, BiNNs are expected to play an important role in energy-efficient and hardware-friendly deep learning.

A problem with BiNNs is that their training is much more difficult than their continuous counterpart.
BiNNs obtained by quantizing already trained DNNs do not work well, and it is preferable to optimize for binary weights directly.
Such training is challenging because it involves a discrete optimization problem. Continuous optimization methods such as the Adam optimizer \citep{kingma2014adam} are not expected to perform well or even converge.

Despite such theoretical issues, a method called Straight-Through-Estimator (STE) \citep{bengio2013estimating}, which employs continuous optimization methods, works remarkably well \citep{courbariaux2015binaryconnect}.
The method is justified based on ``latent'' real-valued weights which are discretized at every iteration to get binary weights.
The gradients used to update the latent weights, however, are computed at the binary weights (see \autoref{fig:binn_loop} (a) for an illustration). It is not clear why these gradients help the search for the minimum of the discrete problem \citep{yin2019understanding, alizadeh2018empirical}. 
Another recent work by \citet{helwegen2019latent} dismisses the idea of latent weights, and proposes a new optimizer called Binary Optimizer (Bop) based on \emph{inertia}.
Unfortunately, the steps used by their optimizers too are derived based on intuition and are not theoretically justified using an optimization problem.
Our goal in this paper is to address this issue and propose a principled approach to justify the algorithmic choices of these previous approaches.

\begin{figure*}[t]
\begin{center}
\includegraphics[scale=0.44]{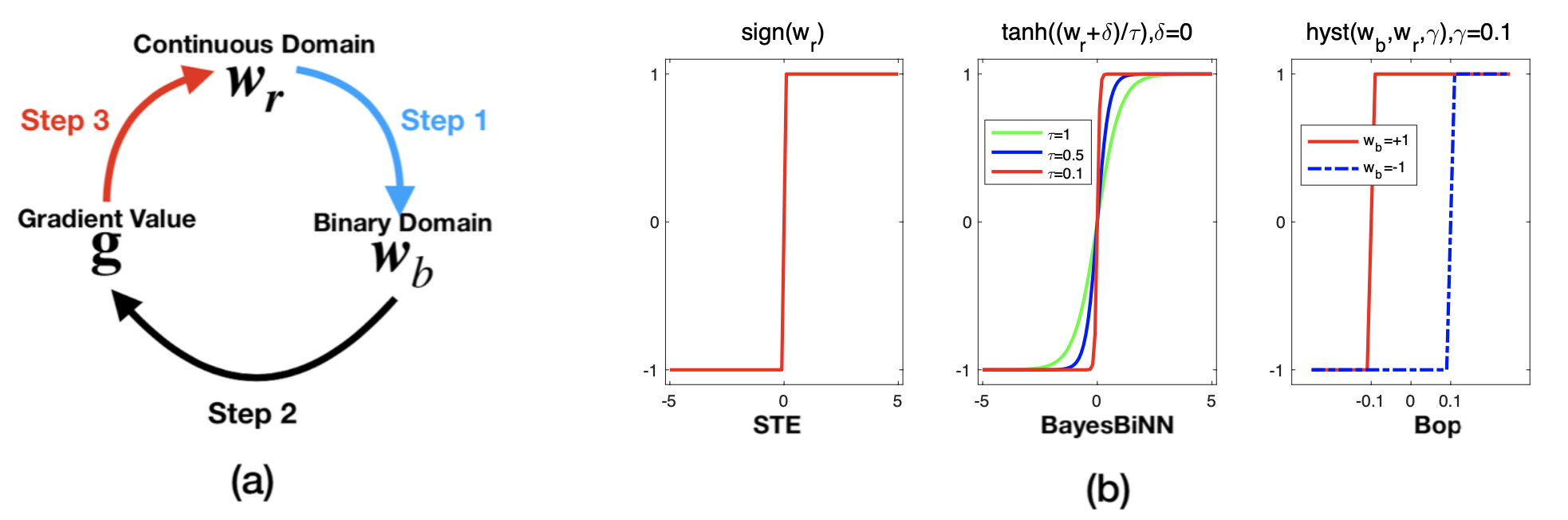}
\end{center}
\caption{ (a): The three steps involved in training BiNNs. In step 1, we obtain binary weights $\vw_b$ from the real-valued parameters $\vw_r$. In step 2, we compute gradients at $\vw_b$ and, in step 3, update $\vw_r$. (b): Different functions used to convert continuous parameters to binary weights. From left to right: Sign function used in STE; Tanh function used in BayesBiNN; Hysteresis function (see (\ref{eq:hyst})) used in Bop.}
\label{fig:binn_loop}
\end{figure*}

\begin{table*}[!htbp] 
\begin{center}
\begin{tabular}{@{}l|l|l|l@{}}
\toprule
& \textbf{STE} & \textbf{Our BayesBiNN method} & \textbf{Bop} \tabularnewline
\midrule
Step 1: Get $\vw_b$ from $\vw_r$ & \textcolor{blue}{ $\vw_{b} \leftarrow \text{sign}(\vw_r)$ } &  \textcolor{blue}{$\vw_{b} \leftarrow \tanh{\left((\vw_r+\boldsymbol{\delta})/\tau\right)}$}  &  \textcolor{blue}{$\vw_{b} \leftarrow \text{hyst}(\vw_r,\vw_{b}, \gamma)$} \\

Step 2: Compute gradient at $\vw_b$ & $\mathbf{\vg} \leftarrow 
         \nabla_{{\vw_{b}}} \loss(y, f_{\vw_{b}}(\vx))$ &
         $\mathbf{\vg} \leftarrow 
         \nabla_{{\vw_{b}}} \loss(y, f_{\vw_{b}}(\vx))$ & 
         $\mathbf{\vg} \leftarrow  \nabla_{{\vw_{b}}} \loss(y, f_{\vw_{b}}(\vx))$ \\
        
Step 3: Update $\vw_r$ & \textcolor{red} {$\vw_r \leftarrow  \vw_r - \alpha\mathbf{\vg}$} &  \textcolor{red}{$\vw_r \leftarrow  (1-\alpha)\vw_r - \alpha\mathbf{s} \odot \mathbf{\vg}$} & \textcolor{red}{$\vw_r \leftarrow  (1-\alpha)\vw_r - \alpha \mathbf{\vg}$} \tabularnewline

\bottomrule
\end{tabular}
\end{center}
\caption{This table compares the steps of our algorithm BayesBiNN to the two existing methods, STE \citep{bengio2013estimating} and Bop \citep{helwegen2019latent}. Here, $\vw_b$ and $\vw_r$ denote the binary and real-valued weights. For step 1, where $\vw_b$ are obtained from $\vw_r$, STE uses the sign of $\vw_r$ while BayesBiNN uses a $tanh$ function with a small noise $\delta$ sampled from a Bernoulli distribution and a temperature parameter $\tau$. As \autoref{fig:binn_loop} (b) shows, as $\tau$ goes to 0, Step 1 of BayesBiNN becomes equal to that of STE. Step 1 of Bop uses the \emph{hysteresis} function shown in \autoref{fig:binn_loop} (b) and becomes similar to sign function as the threshold $\gamma$ goes to 0 (it is flipped but the sign is irrelevant for binary variables). Step 2 is the same for all algorithms. Step 3 of BayesBiNN is very similar to Bop, except that a scaling $\mathbf{s}$  is used (which is similar to the adaptive learning rate algorithms); see (\ref{eq:scaling_s_def}) in Section \ref{sec:BayesBiNN}.}
\label{table:Psedo-code}
\end{table*}

In this paper, we present a Bayesian perspective to justify previous approaches.
Instead of optimizing a discrete objective, the Bayesian approach relaxes it by using a distribution over the binary variable, resulting in a principled approach for discrete optimization.
We use a Bernoulli approximation to the posterior and estimate it using a recently proposed approximate Bayesian inference method called the Bayesian learning rule \citep{khan2017conjugate,emti2020bayesprinciple}.
This results in an algorithm which justifies some of the algorithmic choices made by existing methods; see \autoref{table:Psedo-code} for a summary of results.
Since our algorithm is based on a well-defined optimization problem, it is easier to extend its application.
We show an application for continual learning to avoid catastrophic forgetting \cite{Kirkpatrick2016OvercomingCF}. 
To the best of our knowledge, there is no other work on continual learning of BiNNs so far, perhaps because extending existing methods, like STE, for such tasks is not trivial.
Overall, our work provides a principled approach for training BiNNs that justifies and extends previous approaches. The code to reproduce the results is available at \url{https://github.com/team-approx-bayes/BayesBiNN}.

\subsection{Related Works}
There are two main directions on the study of BiNNs: one involves the design of special network architecture tailored to binary operations \citep{courbariaux2015binaryconnect, rastegari2016xnor, lin2017towards, bethge2020meliusnet} and the other is on the training methods. The latter is the main focus of this paper.

Our algorithm is derived using the Bayesian learning rule recently proposed by \citet{khan2017conjugate,emti2020bayesprinciple}, which is obtained by optimizing the Bayesian objective by using natural gradient descent \citep{amari1998natural,hoffman2013stochastic,khan2017conjugate}. It is shown in \citet{emti2020bayesprinciple} that the Bayesian learning rule can be used to derive and justify many existing learning-algorithms in fields such as optimization, Bayesian statistics, machine learning and deep learning. 
In particular, the Adam optimizer can also be derived as a special case \citep{khan2018fasta,osawa2019practical}.
Our application is yet another example where the rule is used to justify existing algorithms that perform well in practice but whose mechanisms are not well understood.

Instead of using the Bayesian learning rule, it is possible to use other types of variational inference methods, e.g., \citet{shayer2017learning,peters2018probabilistic} used a  \emph{variational optimization} approach \citep{Staines2019variationaloptimization} along with the local reparameterization trick. The Gumbel-softmax trick \citep{maddison2016concrete,jang2016categorical} is also used in \citet{louizos2018relaxed} to train BiNNs. However, instead of specifying a Bernoulli distribution over the weights, \citet{louizos2018relaxed} construct a noisy quantizer and their optimization objective is different from ours. Unlike our method, none of these methods \citep{shayer2017learning,peters2018probabilistic, louizos2018relaxed} result in an update similar to either STE or Bop.

\section{Training Binary Neural Networks (BiNNs)}
Given $\data=\left\{\left(\vx_{i},y_{i}\right)\right\} _{i=1}^{N}$, the goal is to train a neural network $f_{\vw}(\vx)$ with binary weights $\vw \in \left\{-1,+1\right\}^{W}$, where $W$ is the number of weights. The challenge is in optimizing the following discrete optimization objective:
\begin{equation}
\underset{\vw \in\left\{ -1,1\right\} ^{W}}{\min} \sum_{i\in\data}\loss(y_i, f_\vw(\vx_i)),
\label{eq:discrete_optimization}
\end{equation}
where $\loss(y_i, \hat{y}_i)$ is a loss function, e.g., cross-entropy loss for the model predictions $\hat{y}_i := f_{\vw}(\vx_i)$.
It is clear that binarized weights obtained from pre-trained NNs with real-weights do not minimize (\ref{eq:discrete_optimization}) and therefore not expected to give good performance.
Optimizing the objective with respect to binary weights is difficult since gradient-based methods cannot be directly applied.
The gradient of the real-valued weights are not expected to help in the search for the minimum of a discrete objective \citep{yin2019understanding}.

Despite such theoretical concerns, the Straight-Through Estimator (STE) \citep{bengio2013estimating}, which utilizes gradient-based methods, works extremely well.
There have been many recent works that build upon this method, including BinaryConnect \citep{courbariaux2015binaryconnect}, Binarized neural networks \citep{courbariaux2016binarized}, XOR-Net \citep{rastegari2016xnor}, as well as the most recent MeliusNet \citep{bethge2020meliusnet}.
The general approach of such methods is shown in \autoref{fig:binn_loop} in three steps. 
In step 1, we obtain binary weights $\vw_b$ from the real-valued parameters $\vw_r$. In step 2, we compute gradients at $\vw_b$ and, in step 3, update $\vw_r$ using the gradients. 
STE makes a particular choice for the step 1 where a sign function is used to obtain the binary weights from the real-valued weights (see \autoref{table:Psedo-code} for a pseudocode).
However, since the gradient of the sign function with respect to $\vw_r$ is zero almost everywhere, it implies that $\nabla_{\vw_r}\approx \nabla_{\vw_{b}}$. This approximation can be justified in some settings \citep{yin2019understanding} but in general the reasons behind its effectiveness are unknown.

Recently \citet{helwegen2019latent} proposed a new method that goes against the justification behind STE. They argue that ``latent'' weights used in STE based methods do not exist. Instead, they provide a new perspective: the sign of each element of $\vw_r$ represents a binary weight while its magnitude encodes some inertia against flipping the sign of the binary weight. With this perspective, they propose the Binary optimizer (Bop) method which keeps track of an exponential moving average \citep{gardner1985exponential} of the gradient $\mathbf{\vg}$ during the training process and then decide whether to flip the sign of the binary weights when they exceed a certain threshold $\gamma$. The Bop algorithm is shown in \autoref{table:Psedo-code}. However, derivation of Bop is also based on intuition and heuristics. It remains unclear why the exponential moving average of the gradient is used in Step 3 and what objective the algorithm is optimizing. The selection of the threshold $\gamma$ is another difficult choice in the algorithm.   

Indeed, Bayesian methods do present a principled way to incorporate the ideas used in both STE and Bop. For example, the idea of ``generating'' binary weights from real-valued parameters can be though of as sampling from a discrete distribution with real-valued parameters. In fact, the sign function used in STE is related to the ``soft-thresholding'' used in machine learning. Despite this there exist no work on Bayesian training of BiNNs that can give an algorithm similar to STE or Bop. In this work, we fix this gap and show that, by using the Bayesian learning rule, we recover a method that justifies some of the steps of STE and Bop, and enable us to extend their application. We will now describe our method.

\section{BayesBiNN: Binary NNs with Bayes}
\label{sec:BayesBiNN}

We will now describe our approach based on a Bayesian formulation of the discrete optimization problem in (\ref{eq:discrete_optimization}). A Bayesian formulation of a loss-based approach can be written as the following minimization problem with respect to a distribution $q(\vw)$ \citep{zellner1988optimal,bissiri2016general}
\begin{align}
\myexpect_{q(\vw)} \sqr{ \sum_{i=1}^N\loss(y_i, f_{\vw}(\vx_i)) }  + \dkls{}{q(\vw)}{p(\vw)}, \label{eq:ELBO}
\end{align}
where $p(\vw)$ is a prior distribution and $q(\vw)$ is the posterior distribution or its approximation. The formulation is general and does not require the loss to correspond to a probabilistic model. When the loss  indeed corresponds to a negative log likelihood function, this minimization results in the posterior distribution which is equivalent to Bayes' rule \cite{bissiri2016general}. When the space of $q(\vw)$ is restricted, this results in an approximation to the posterior, which is then equivalent to variational inference \citep{jordan1999introduction}.
For our purpose, this formulation enables us to derive an algorithm that resembles existing methods such as STE and Bop.

\subsection{BayesBiNN optimizer}
To solve the optimization problem (\ref{eq:ELBO}), the Bayesian learning rule \citep{emti2020bayesprinciple} considers a class of minimal exponential family distributions \citep{WainwrightJordan08}
\begin{align}
q\left(\vw\right) :=  h\left(\vw\right)\exp{\left[\vlambda^T\phi(\vw)-A(\vlambda)\right]},
\label{eq:exponential_family}
\end{align}
where $\vlambda$ is the natural parameter, $\phi(\vw)$ is the vector of sufficient statistics, $A(\vlambda)$ is the log-partition function, and $h\left(\vw\right)$ is the base measure.
When the prior distribution $p(\vw)$ belongs to the same exponential-family as $q\left(\vw\right)$ in (\ref{eq:exponential_family}), and the base measure $h(\vw)=1$, the Bayesian learning uses the following update of the natural parameter \citep{khan2017conjugate,emti2020bayesprinciple}
\begin{equation}
\scalemath{0.86}{
\boldsymbol{\lambda} \leftarrow (1-\alpha)\boldsymbol{\lambda}- \alpha\left\{\nabla_{\vmu}\myexpect_{q(\vw)}\left[\sum_{i=1}^N\loss(y_i, f_{\vw}(\vx_i))\right]-\vlambda_0\right\}
},
\label{eq:BayesLearningRule}
\end{equation}
where $\alpha$ is the learning rate,  $\vmu=\myexpect_{q(\vw)}\sqr{\phi\left(\vw\right)}$ is the expectation parameter of $q\left(\vw\right)$, and $\vlambda_0$ is the natural parameter of the prior distribution $p(\vw)$, which is assumed to belong to the same exponential-family as $q(\vw)$. The Bayesian learning rule is a natural gradient algorithm \citep{khan2017conjugate,emti2020bayesprinciple}. 
An interesting point of this rule is that the gradient is computed with respect to the expectation parameter $\vmu$, while the update is performed on the natural parameter $\boldsymbol{\lambda}$. 
This particular choice leads to an update similar to STE for BiNNs, as we show next.

We start by specifying the form of $p(\vw)$ and $q(\vw)$.
A priori, we assume that the weights are equally likely to be either $-1$ or $+1$, i.e., the prior $p(\vw)$ is a (symmetric) Bernoulli distribution with a probability of $\half$ for each state.
For the posterior approximation $q(\vw)$, we use the mean-field (symmetric) Bernoulli distribution 
\begin{align}
q\left(\vw\right) & =\prod_{j=1}^{W} p_j^{\frac{1+w_{j}}{2}}\left(1-p_j\right)^{\frac{1-w_{j}}{2}} ,
 \label{eq:Ber_distrib}
\end{align}
where $p_j$ is the probability that $w_j=+1$, and $W$ is the number of parameters. 
Our goal is to learn the parameters $p_j$ of the approximations. 
The Bernoulli distribution defined in (\ref{eq:Ber_distrib}) is a special case of the \emph{minimal} exponential family distribution, where the corresponding natural and expectation parameters of each weight $w_i$ are 
\begin{align}
\lambda_{j} & :=\frac{1}{2}\log\frac{p_{j}}{1-p_{j}},\quad\quad
\mu_{j} := 2p_{j}-1. 
\label{eq:natural_param_def}
\end{align}
The natural parameter $\vlambda_0$ of the prior $p(\vw)$ is therefore $\boldsymbol{0}$. Using these definitions, we can directly apply the Bayesian learning rule to learn the posterior Bernoulli distribution of the binary weights. 

In addition to these definitions, we also require the gradient with respect to $\vmu$ to implement the rule (\ref{eq:BayesLearningRule}). A straightforward solution is to use the REINFORCE method \citep{williams1992simple} which transforms the gradient of the expectation into the expectation of the gradient by using the log-derivative trick, i.e., 
\begin{align}
    \nabla_{\vmu}\mathbb{E}_{q(\vw)} \sqr{ \loss(y, f_\vw(\vx))}  =\mathbb{E}_{q(w)} \sqr{ \loss(y,f_w(\vx)) \nabla_\mu \log q(\vw)}. \nonumber
\end{align}
This method, however, does not use the \emph{minibatch gradient} (the gradient of the loss $\loss(y, f_\vw(\vx))$ on a minibatch of examples), which is essential to show the similarity to STE and Bop. The REINFORCE method also suffers from high variance. Due to these reasons, we do not use this method. 

Instead, we resort to another reparameterization trick for discrete variables called the Gumbel-softmax trick \citep{maddison2016concrete,jang2016categorical}, which, as we will see, leads to an update similar to STE/Bop. The idea of Gumbel-softmax trick is to introduce the \emph{concrete distribution} that leads to a \emph{relaxation} of the discrete random variables. Specifically, as shown in Appendix B in \citep{maddison2016concrete}, for a binary variable $w_j \in\{0,1\}$ with $P(w_j=1) = p_j$, we can use the following relaxed variable $w_r^{\epsilon_j,\tau}(p_j) \in (0,1)$:
\begin{align}
    w_r^{\epsilon_j,\tau}(p_j) := \frac{1}{1+\exp\left(-\frac{2\lambda_j+2\delta_j}{\tau}\right)},
    \label{eq:w_relaxed}
\end{align}
where $\tau>0$ is a temperature parameter, $\lambda_j:=\frac{1}{2}\log\frac{p_j}{1-p_j}$ is the natural parameter, and $\delta_j$ is defined as follows,
\begin{align}
\delta_j := \frac{1}{2}\log\frac{\epsilon_j}{1-\epsilon_j},
\label{eq:delta_def}
\end{align}
with $\epsilon_j \sim \mathcal{U}\left(0,1\right)$ sampled from a uniform distribution.
The $w_r^{\epsilon_j,\tau}(p_j)$ are samples from a Concrete distribution which has a closed-form expression \citep{maddison2016concrete}: 
\begin{align}
&p\left(w_r^{\epsilon_j,\tau}(p_j)\right)\\
&\quad:=\frac{\tau e^{2\lambda}(w_r^{\epsilon_j,\tau}(p_j))^{-\tau-1}\left(1-(w_r^{\epsilon_j,\tau}(p_j))\right)^{-\tau-1}}{\left(e^{2\lambda}(w_r^{\epsilon_j,\tau}(p_j))^{-\tau}+\left(1-(w_r^{\epsilon_j,\tau}(p_j))\right)^{-\tau}\right)^{2}}. \nonumber 
\end{align}
Instead of differentiating the objective with respect to binary variables $w_j$, we can differentiate with respect to $w_r^{\epsilon_j,\tau}(p_j)$. We will use this to approximate the gradient with respect to $\vmu$ in terms of the minibatch gradient.

In our case, entries $w_j$ of $\vw$ take value in $\{+1, -1\}$ rather than in $\{0,1\}$. The relaxed version could be obtained by a linear transformation of the concrete variables $w_r^{\epsilon_j,\tau}(p_j)$ in (\ref{eq:w_relaxed}), i.e., 
\begin{align}
w_{b}^{\epsilon_j,\tau}(\lambda_j) :=2w_r^{\epsilon_j,\tau}(p_j)-1
= \tanh{\left((\lambda_j + \delta_j)/\tau\right)},
\label{eq:Repara}
\end{align}
where, unlike (\ref{eq:w_relaxed}), we have explicitly written the dependency in terms of $\lambda_j$ instead of the vector of $p_j$. 
Since $w_{b}^{\epsilon_j,\tau}(\lambda_j)$ are continuous, we can differentiate them with respect to $\mu_j$ by using the chain rule. 
The lemma below states the result where $\vw_b^{\boldsymbol{\epsilon},\tau}(\boldsymbol{\lambda})$ and $\vepsilon$ denote the vectors of $w_{b}^{\epsilon_j,\tau}(\lambda_j)$ and $\epsilon_j$, respectively, for all $j=1,2,\ldots,W$.
\begin{lemma} \label{lemma:gradient}
By using the Gumbel-softmax trick, we get the following approximation in terms of the minibatch gradient:
\begin{align}
\nabla_{\vmu}\myexpect_{q(\vw)}\left[\sum_{i=1}^N\loss(y_i, f_{\vw}(\vx_i))\right] \approx \mathbf{s} \odot \mathbf{\vg},  
\label{eq:GradientApproximation}
\end{align}
where 
\begin{align}
& \mathbf{\vg}  := \frac{1}{M} \sum_{i\in\minibatch} \left. \nabla_{\vw_r} \loss(y_i, f_{\vw_r}(\vx_i)) \right\vert_{ \vw_r = {\vw_b^{\boldsymbol{\epsilon},\tau}(\boldsymbol{\lambda})}}, \label{eq:scaling_s_def} \\
& \mathbf{s} := \frac{N(1-{(\vw_b^{{\boldsymbol{\epsilon}},\tau}(\vlambda)})^2)}{\tau(1-\tanh{(\vlambda)}^2)},
\label{eq:gradient_minibatch}
\end{align}
and $\minibatch$ is a mini-batch of $M$ examples.
\end{lemma}

\begin{proof}
Using (\ref{eq:Repara}), we can first approximate the objective in terms of $\vw_b^{\boldsymbol{\epsilon}, \tau}(\vlambda)$ and $\vepsilon$, and then push the gradient with respect to $\vmu$ inside the expectation as shown below:
\begin{align}
 \nabla_{\vmu}\myexpect_{q(\vw)} &\left[\sum_{i=1}^N\loss(y_i, f_{\vw}(\vx_i))\right] \nonumber \\
 \approx & \myexpect_{q(\boldsymbol{\epsilon})}\left[\sum_{i=1}^N\nabla_{\vmu}\loss(y_i, f_{\vw_b^{\boldsymbol{\epsilon},\tau}(\boldsymbol{\lambda})}(\vx_i))\right]. 
\label{eq:gradient2expectation}
\end{align}

%
Using the chain rule, the $j$-th element of the gradient on the right hand side can be obtained as follows:
\begin{align}
\nabla_{\mu_j}\loss(y_i, f_{\vw_b^{\boldsymbol{\epsilon}, \tau}}(\vx_i)) = \nabla_{w_{b}^{\epsilon_j,\tau}}\loss(y_i, f_{\vw_b^{\boldsymbol{\epsilon}, \tau}}(\vx_i))\frac{d{w_{b}^{\epsilon_j, \tau}(\lambda_j)}}{d\mu_j}.
\end{align}
According to the definition of natural parameter and expectation parameter in (\ref{eq:natural_param_def}), we have $\lambda_j = \frac{1}{2}\log\frac{1 + \mu_j}{1-\mu_j} $, therefore after some algebra we can write: 
\begin{align}
\frac{d{w_{b}^{\epsilon_j, \tau}}}{d\mu_j}
 & =\frac{1-({w}^{\epsilon_j, \tau}_{b}(\lambda_j))^2}{\tau\left(1-\tanh^2(\lambda_j)\right)}.
\label{eq:dz2dmui}
\end{align}
By using a mini-batch $\minibatch$ of $M$ examples and one sample $\boldsymbol{\epsilon}$, (\ref{eq:gradient2expectation})-(\ref{eq:dz2dmui}) give us (\ref{eq:GradientApproximation}). \qed
\end{proof}
Substituting the result of Lemma \ref{lemma:gradient} into the Bayesian learning rule in  (\ref{eq:BayesLearningRule}), we obtain the following update:
\begin{align}
\vlambda \leftarrow  (1-\alpha)\vlambda - \alpha \left[\mathbf{s} \odot \mathbf{\vg}-\vlambda_0\right].
\label{eq:lamda_update}
\end{align}
The resulting optimizer, which we call BayesBiNN, is shown in \autoref{table:Psedo-code}, where we assume $\vlambda_0=\mathbf{0}$ (since the probability of $w_i=+1$ is 1/2 a priori). For the ease of comparison with other methods, the natural parameter $\boldsymbol{\vlambda}$ is replaced with continuous variables $\vw_r$. 

At test-time, we can either use the predictions obtained using Monte-Carlo average (which we refer to  as the ``mean''):
\begin{align}
 \hat{p}_{k} = \frac{1}{C}\sum_{c=1}^{C}p\left(y=k|\vx,\vw^{(c)}\right),   
\end{align}
with $\vw^{(c)} \sim q(\vw)$ and $C$ is the number of samples, or the predictions obtained by using the mode $\hat{\vw }$ of
 of $q(\vw)$:
$\hat{p}_{k} =  p\left(y=k|\vx,\hat{\vw}\right)$,
where $\hat{\vw } = \textrm{sign}(\tanh{\left(\vlambda\right)})$ (which we refer to as the ``mode'').

\subsection{Justification of Previous Approaches} \label{sec:justification}
In this section, we show how BayesBiNN justifies the steps of STE and Bop. A summary is shown in \autoref{table:Psedo-code}. First, BayesBiNN justifies the use of gradient based methods to solve the discrete optimization problem (\ref{eq:discrete_optimization}). As opposed to (\ref{eq:discrete_optimization}), the new objective in (\ref{eq:ELBO}) is over a continuous parameter $\vlambda$ and thus gradient descent can be used. The underlying principle is similar to stochastic relaxation for non-differentiable optimization \citep{lemarechal1989nondifferentiable, geman1984stochastic}, evolution strategies \citep{huning1976evolutionsstrategie}, and variational optimization \citep{Staines2019variationaloptimization}. 

Second, some of the algorithmic choices of previous methods such as STE and Bop are justified by BayesBiNN. Specifically, when the temperature $\tau$ in BayesBiNN is small enough, the $\tanh{(\cdot)}$ function in \autoref{table:Psedo-code} behaves like the $\text{sign}(\cdot)$ function used in STE; see \autoref{fig:binn_loop} (b).  From this perspective, the latent weights $\mathbf{\vw}_{r}$ in STE play a similar role as the natural parameter $\boldsymbol{\lambda}$ of BayesBiNN. In particular, when there is no sampling, i.e., $\delta \leftarrow 0$ in BayesBiNN, the two algorithms will become very similar to each other. BayesBiNN justifies the step 1 used in STE by using the Bayesian perspective. 

BayesBiNN also justifies step 3 of Bop\footnote{Note that the step 3 of Bop in \autoref{table:Psedo-code} is an equivalent but ``flipped'' version of the one used by \citet{helwegen2019latent}; see \autoref{Sec:hyst-function} for details.} in \autoref{table:Psedo-code}, where an exponential moving average of gradients is used. This is referred to as the momentum term in Bop which plays a similar role as the natural parameter in BayesBiNN. In \citet{helwegen2019latent}, the momentum is interpreted as a quantity related to \emph{inertia}, which indicates the strength of the state of weights. Since the natural parameter in the binary distribution (\ref{eq:Ber_distrib}) essentially indicates the strength of the probability being $-1$ or $+1$ for each weight, BayesBiNN provides an alternative explanation for Bop.

A recent mirror descent view  proposed in \citet{ajanthan2019mirror} also interprets the continuous parameters as the dual of the quantized ones. As there is an equivalence between the natural gradient descent and mirror descent \citep{raskutti2015information,khan2017conjugate}, the proposed BayesBiNN also provides an interesting perspective on the mirror descent framework for BiNNs training.

\subsection{Benefits of BayesBiNN}
Apart from justifying previous methods, BayesBiNN has several other advantages. First, since its algorithmic form is similar to existing optimizers, it is very easy to implement BayesBiNN by using existing codebases. Second, as a Bayesian method, BayesBiNN provides uncertainty estimates, which can be useful for decision making.  
The uncertainty obtained using BayesBiNN can enable us to perform continual learning by using the variational continual learning (VCL) framework \citep{nguyen2017variational}, as we discuss next.

In continual learning, our goal is to learn the parameters of the neural network from a
set of sequentially arriving datasets $\data_{t}=\left\{ \left(\vx_{i},y_{i}\right)\right\} _{i=1}^{N_t},t=1,2,\ldots,T$. While training the $t$-th task with dataset $\data_t$, we do not have access to the datasets of past tasks, i.e., $\data_1,\data_2,\ldots,\data_{t-1}$. Training on task $\data_{t}$ using a deep-learning optimizer usually leads to a huge performance loss on the past tasks \citep{Kirkpatrick2016OvercomingCF}. The goal of continual learning is to fix such catastrophic forgetting of the past. 

For full-precision networks, a common approach to solve this problem is to use weight-regularization, e.g., the \emph{elastic weight consolidation} (EWC) method \citep{Kirkpatrick2016OvercomingCF} uses a Fisher information matrix to regularize the weights: 
\begin{align}
\sum_{i \in \data_t }\loss(y_i, f_{\vw}(\vx_i)) + \varepsilon(\vw - \vw_{t-1})^T{\mathbf{F}_t}(\vw - \vw_{t-1}), \label{eq:EWC_loss}
\end{align}
where $\varepsilon$ is a regularization parameter and $\mathbf{F}_t$ is the Fisher information matrix at $\vw_t$. The hope is to keep the new weights close to the old weights, but in a discrete optimization problem, it is impossible to characterize such closeness using a quadratic regularizer as above. Therefore, it is unclear why such a regularizer will be useful. In addition, the use of the Fisher information matrix $\mathbf{F}_t$ typically assumes that the weights are continuous and the matrix does not provide a meaningful quantity for discrete weights. For these reasons, extending existing approaches such as STE and Bop to continual learning is a nontrivial task.

Fortunately, for BayesBiNN, this is very easy because the objective function is well-defined. We use the VCL framework \citep{nguyen2017variational} where we regularize the distributions instead of the weights. The Kullback-Leibler term is used as the regularizer. Denoting by $q_{t-1}(\vw)$ the posterior distribution at task $t-1$, we can replace the prior distribution $p(\vw)$ in (\ref{eq:ELBO}) by $q_{t-1}(\vw)$:
\begin{equation}
\scalemath{0.87}{
\myexpect_{q_{t}(\vw)} \left[ \sum_{i \in \data_t}\loss(y_i, f_w(\vx_i))\right] 
+ \dkls{}{q_{t}(\vw)}{q_{t-1}(\vw)}
}.
\label{eq:ELBO_VCL}
\end{equation}
This leads to a slight modification of the update in BayesBiNN where the prior natural parameter $\vlambda_0$ in (\ref{eq:lamda_update}) is replaced by the natural parameter $\vlambda_{t-1}$ of $q_{t-1}(\vw)$. The new update of the natural parameter $\vlambda_t$ of $q_t(\vw)$ is shown below: 
\begin{align}
\vlambda_t \leftarrow  (1-\alpha)\vlambda_t - \alpha \left[\mathbf{s} \odot \mathbf{\vg} -\vlambda_{t-1} \right].
\label{eq:lamda_update_CL}
\end{align}
By using a posterior approximation $q(\vw)$ and a well-defined objective, BayesBiNN enables the application of STE/Bop like methods to such challenging continual learning problems.

\begin{figure}[!t]
\centering
\includegraphics[width=1\columnwidth]{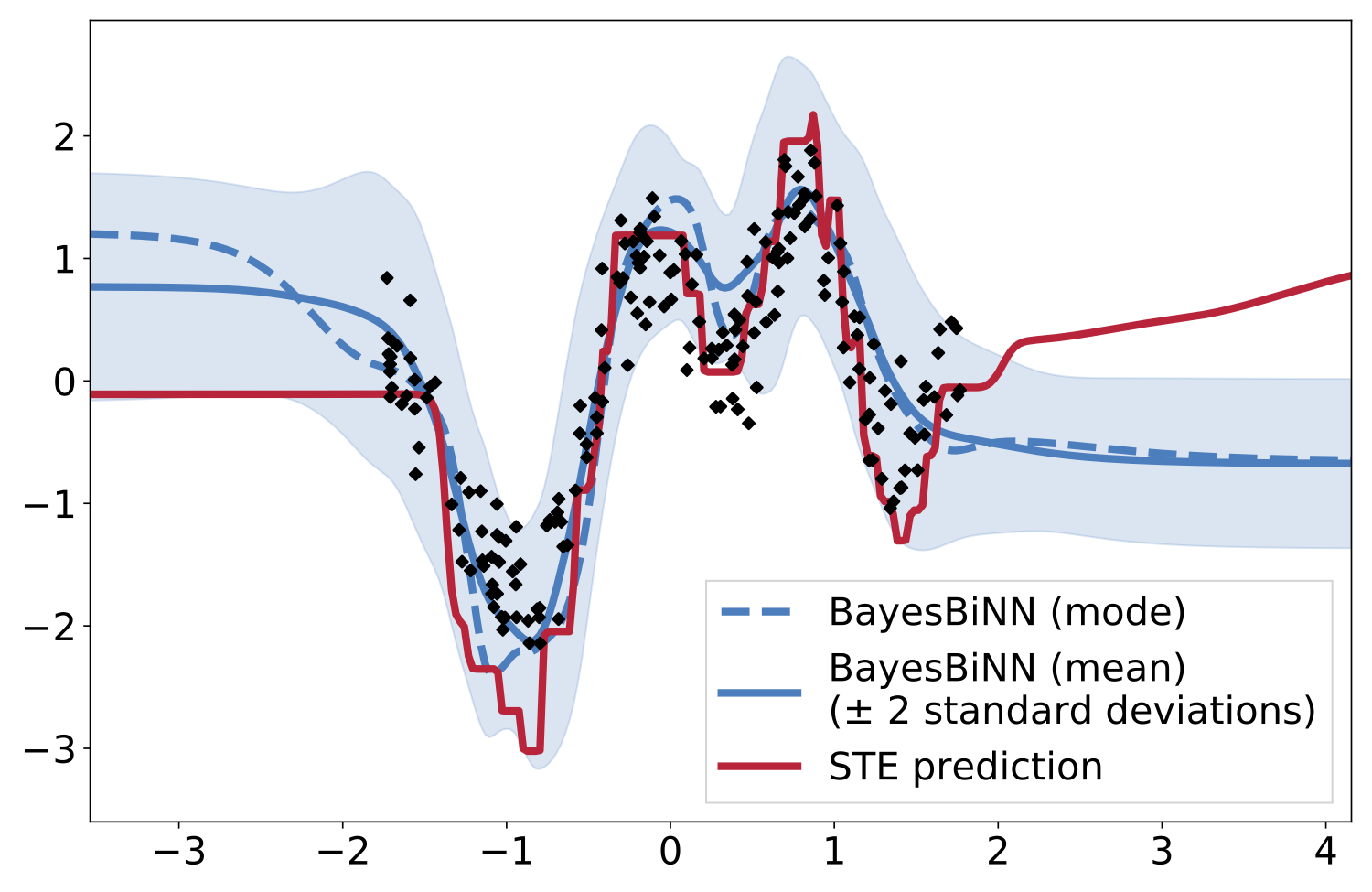}
\caption{Regression on Snelson dataset \citep{snelson2006sparsegp}.  BayesBiNN (mean) gives much smoother curve than the STE. Uncertainty is low in  areas with  plenty of data points.}
\label{fig:b2n2_regression_snelson}
\end{figure}

\section{Experimental Results}
In this section, we present numerical experiments to demonstrate the performance of BayesBiNN on both synthetic and real image data for different kinds of neural network architectures.  We also show an application of BayesBiNN to continual learning. The code to reproduce the results is available at \url{https://github.com/team-approx-bayes/BayesBiNN}. 
\label{ExperimentalSec}

\begin{figure*}[t]
\centering
\includegraphics[width=0.98\textwidth]{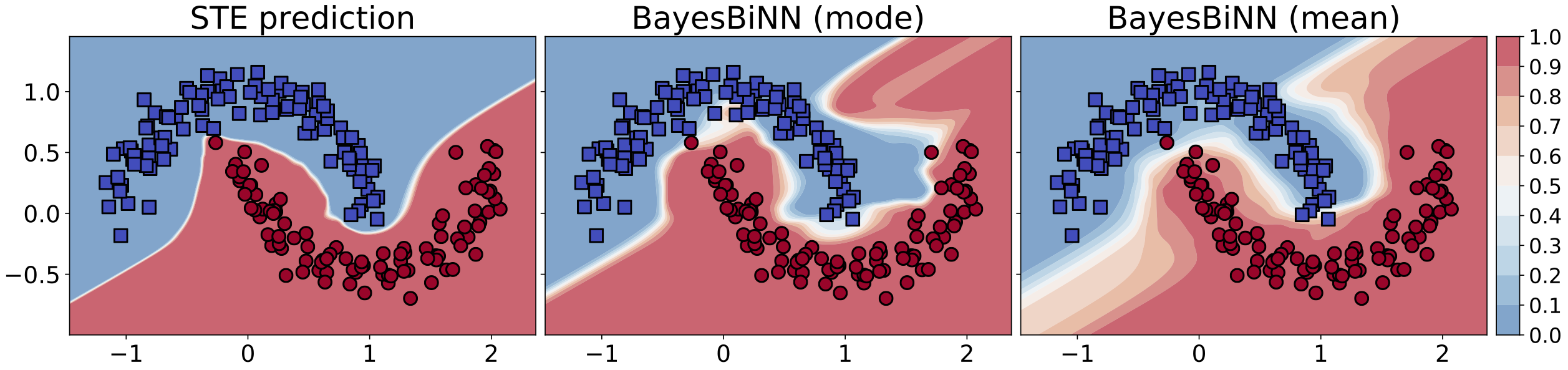}
\caption{Classification on the two moons dataset with different optimizers. From left to right:  STE, BayesBiNN using the mode, BayesBiNN using the predictive mean from 10 posterior Monte Carlo samples. STE is more overconfident than BayesBiNN, and BayesBiNN (mean) gives reasonable uncertainty in regions further away from the data.}
\label{fig:b2n2_vs_ste}
\end{figure*}

\subsection{Synthetic Data}
\label{UncertaintySec}
First, we present visualizations on toy  regression and binary classification  problems.  The STE~\citep{bengio2013estimating} algorithm is used as a baseline for which we employ Adam for training. We use a multi-layer perceptron (MLP) with two hidden layers of 64 units and $\tanh{}$ activation functions.

\paragraph{Regression} 
In \autoref{fig:b2n2_regression_snelson}, we show results for regression on the Snelson dataset~\citep{snelson2006sparsegp}. For this experiment, we add a Batch normalization (BN) \citep{ioffe2015batch} layer (but no learned gain or bias terms) after the last fully connected layer. As seen in Figure \ref{fig:b2n2_regression_snelson}, predictions obtained using `BayesBiNN (mean)' gives much smoother curves than STE, as expected. Uncertainty is lower in the areas with little noise and plenty of data points compared to areas with no data. Experimental details of the training process are provided in the Appendix~\ref{sec:settings_uncertainty}.

\begin{table*}[t]
\caption{Test accuracy of different optimizers for MNIST, CIFAR-10 and CIFAR-100 (Averaged over 5 runs). In all the three benchmark datasets, BayesBiNN achieves similar performance as STE Adam and closely approaches the performance of full-precision networks.}
\label{tab:all_results}
\begin{center}
\resizebox{1.5\columnwidth}{!}{%
\begin{tabular}{@{}llllll@{}}
\toprule
& {\bf Optimizer} & {\bf MNIST}  & {\bf CIFAR-10}  & {\bf CIFAR-100} \\
\midrule
\multirow{5}{*}{}
& STE Adam        &  $\bold{98.85} \pm 0.09$ \%   &  $\bold{93.55} \pm 0.15$ \%   &  $72.89 \pm 0.21$ \% \\
& Bop             & $98.47 \pm 0.02$ \%   & $93.00 \pm 0.11$ \%   &  $69.58 \pm 0.15$ \%  \\
& PMF             &  $98.80 \pm 0.06$ \%      &  $91.43 \pm 0.14$ \%   &  $70.45 \pm 0.25$ \% \\
& \textbf{BayesBiNN} (mode)  &  $\bold{98.86} \pm 0.05$ \%    &  $\bold{93.72} \pm0.16$ \%  & $\bold{73.68} \pm0.31$ \% \\

& \textbf{BayesBiNN} (mean)  &   $\bold{98.86} \pm 0.05$ \%   &  $\bold{93.72} \pm0.15$ \%   & $\bold{73.65} \pm0.41$ \%   \\

\cmidrule(l){2-5}
& Full-precision  &  $99.01 \pm 0.06$ \%   &  $93.90 \pm0.17$ \%   &  $74.83 \pm0.26$ \% \\

\bottomrule
\end{tabular}
}
\end{center}
\end{table*}

\paragraph{Classification} 
\autoref{fig:b2n2_vs_ste} shows STE and BayesBiNN on the two moons dataset \citep{two_moons_dataset} with 100 data points in each class.
STE (the leftmost figure) gives a point estimate of the weights and results in a fairly deterministic classifier.
When using the mode of the BayesBiNN distribution (the middle figure), the results are similar, with the fit being slightly worse but overall less overconfident, especially in the regions with no data.
Using the mean over 10 samples drawn from the posterior distribution $q(\vw)$ (the rightmost figure), we get much better uncertainty estimates as we move away from the data. Experimental details of the training process are provided in the Appendix~\ref{sec:settings_uncertainty}.

\subsection{Image Classification on Real Datasets}
\label{sec:image-classification}
We now present results on three benchmark real datasets widely used for image classification: MNIST \citep{mnist_dataset}, CIFAR-10 \citep{krizhevsky2009learning} and CIFAR-100 \citep{krizhevsky2009learning}. 
We  compare to three other optimizers\footnote{We use Bop code available at \url{https://github.com/plumerai/rethinking-bnn-optimization} and PMF code available at \url{https://github.com/tajanthan/pmf}.}: STE \citep{bengio2013estimating} using Adam with weight clipping and gradient clipping \citep{courbariaux2015binaryconnect,maddison2016concrete,alizadeh2018empirical}; latent-free Bop \citep{helwegen2019latent}; and the proximal mean-field (PMF) \citep{Ajanthan2019proximal}. An additional comparison with the LR-net method of \citet{shayer2017learning} is given in Appendix ~\ref{sec:LR-net}. 
For a fair comparison, we keep all conditions the same except for the optimization methods themselves. For our proposed BayesBiNN, we report results using both the mode and the mean. For all the experiments, standard categorical cross-entropy loss is used and we take 10\% of the training set for validation and report the best accuracy on the test set corresponding to the highest validation accuracy achieved during training. 

For MNIST, we use a multilayer perceptron (MLP) with three hidden layers with 2048 units and rectified linear units (ReLU) \citep{alizadeh2018empirical} activations. Both Batch normalization (BN) \footnote{Here the parameters of BN layers are not learned. However, they could also be learned by applying a conventional optimizer such as Adam separately, which is easy to implement.} \citep{ioffe2015batch} and dropout are used. No data augmentation is performed. For CIFAR-10 and CIFAR-100, we use the  BinaryConnect CNN network in \citet{alizadeh2018empirical}, which is a VGG-like structure similar to the one used in \citet{helwegen2019latent}. Standard data augmentation is used \citep{graham2014spatially}, where 4 pixels are padded on each side, a random 32 $\times$ 32 crop is applied, followed by a random horizontal flip. Note that no ZCA whitening is used as in \citet{courbariaux2015binaryconnect, alizadeh2018empirical}. The details of the experimental setting,  including the detailed network architecture and values of all hyper-parameters, are provided in Appendix~\ref{sec:settings_mnist_cifar} in the supplementary material. 

As shown in \autoref{tab:all_results}, the proposed BayesBiNN achieves similar performances (slightly better for CIFAR-100) as STE Adam in all the three benchmark datasets and approaches the performance of full-precision DNNs. The detailed results, such as the train/validation accuracy as well as the training curves are provided in Appendix~\ref{sec:settings_mnist_cifar} in the supplementary material.    

\begin{figure*}[htbp]
\centering
\includegraphics[width=1.0\textwidth]{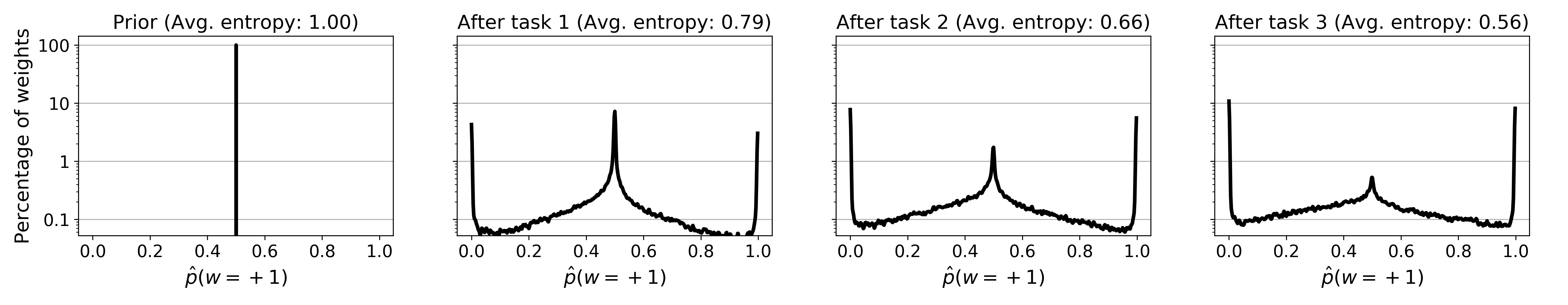}
\caption{Evolution of the distribution over the binary weights during the learning process in continual learning. The histogram of the weight probabilities ($\hat{p}\left(w_{j}=+1\right)$ for all $j$) is shown after learning on different tasks. At the very beginning, all the weights are equal to -1 or +1 with prior probability 0.5 and thus have maximum average entropy 1.0. As the number of learned tasks increases, the distribution spreads and becomes flatter, implying that the average entropy of the binary weights decreases, i.e., the weights of BiNNs become more and more deterministic.}
\label{fig:CL_evolution}
\end{figure*}

\subsection{Continual learning with binary neural networks}

We now show an application of BayesBiNN to continual learning. We consider the popular benchmark of permuted MNIST \citep{goodfellow2013empirical,Kirkpatrick2016OvercomingCF,nguyen2017variational,zenke2017continual}, where each dataset $\data_t$ consists of labeled MNIST images whose pixels are permuted randomly. Similar to \citet{nguyen2017variational}, we use a fully connected single-head network with two hidden layers containing 100 hidden units with ReLu activations. No coresets are used.  The details of the experiment, e.g., the network architecture and values of hyper-parameters, are provided in Appendix~\ref{sec:settings_continual_learning}.
\begin{figure}[h!]
\begin{center}
\includegraphics[width=1.0\columnwidth]{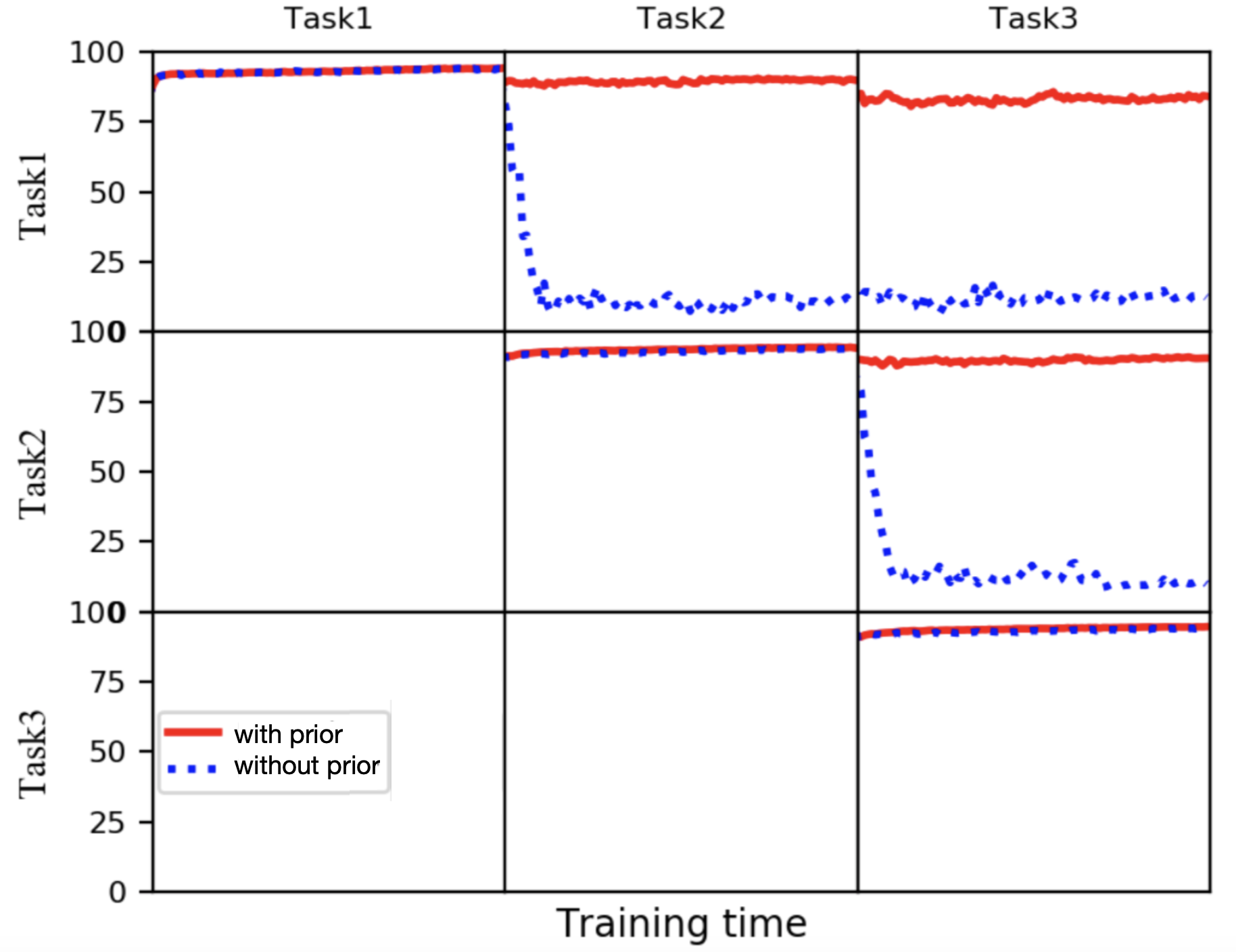}
\end{center}
\caption{Test accuracy curves of continual learning for BiNNs using BayesBiNN on permuted MNIST. The neural network is trained with or without prior for BayesBiNN, respectively. Specifically, with prior (red solid line) indicates using BayesBiNN with the posterior approximation $q_{t-1}(\vw)$ as the prior for task $t$ while without prior (blue dotted line) indicates using BayesBiNN with fixed prior where $\vlambda_0 = \boldsymbol{0}$. The test accuracy on the test set is averaged over 5 random runs. The X-axis shows the training time (epochs) and Y-axis shows the average test accuracy of different tasks as the training time increases.}
\label{fig:cl_results_bvi}
\end{figure}

As shown in \autoref{fig:cl_results_bvi}, using BayesBiNN with the posterior approximation $q_{t-1}(\vw)$ as the prior for task $t$ (red solid line), we achieve significant improvements in overcoming catastrophic forgetting of the past.
When the prior is fixed to be $\vlambda_0=\boldsymbol{0}$ (blue dotted line), the network performs badly on the past tasks, e.g., in the top row, after trained on task 2 and task 3, the network performs badly on the previous task 1.
The reason for better performance when using the posterior approximation  $q_{t-1}(\vw)$ as the prior for task $t$ is directly related to the uncertainty estimated by the posterior approximation $q_{t-1}(\vw)$. 
To visualize the uncertainty, \autoref{fig:CL_evolution} shows the histogram of the weight probabilities ($\hat{p}\left(w_{j}=+1\right)$ for all $j$). The prior probability, shown in the first plot, is set to $\frac{1}{2}$ for all weights (entropy is 1.0). As we train on more tasks, the uncertainty decreases and the weights of BiNNs become more and more deterministic (the distribution spreads and becomes flatter).
As desired, with more data, the network reduces the entropy of the distribution and the uncertainty is useful to perform continual learning.


\section{Conclusion}
Binary neural networks (BiNNs) are computation-efficient and hardware-friendly, but their training is challenging since in theory it involves a  discrete  optimization  problem. However, some gradient-based methods such as the STE work quite well in practice despite ignoring the discrete nature of the optimization problem, which is surprising and there is a lack of principled justification of their success. In this paper, we proposed a principled approach to train the binary neural networks using the Bayesian learning rule. The resulting optimizer, which we call BayesBiNN, not only justifies some of the algorithmic choices made by existing methods such as the STE and Bop but also facilitates the extensions of them, e.g., enabling uncertainty estimation for continual learning to avoid the catastrophic forgetting problem for binary neural networks. 

\newpage
\section*{Acknowledgements}
We would like to thank the anonymous reviewers for their valuable comments. X. Meng would like to thank Milad Alizadeh (Oxford University) and Koen Helwegen (Plumerai Research) for useful discussions. We are thankful for the RAIDEN computing system and its support team at the RIKEN Center for Advanced Intelligence Project which we used extensively for our experiments.

\nocite{langley00}

\bibliography{main}
\bibliographystyle{icml2020}



\onecolumn

\icmltitle{Training Binary Neural Networks using the Bayesian Learning Rule: Appendix}

\vskip 0.3in

\appendix
\section{Two equivalent forms of  hysteresis function in Bop}
\label{Sec:hyst-function}
The original update rule and the corresponding definition of the \emph{hysteresis} function $\text{hyst}(\cdot)$ in Bop are \citep{helwegen2019latent}
\begin{align}
\vw_r & \leftarrow  (1-\alpha)\vw_r + \alpha \mathbf{\vg}, \label{eq:update1}\\ 
y & =\textrm{hyst1}\left(w_{r},w_b,\gamma\right)\nonumber \\
 & \equiv \begin{cases}
-w_b & \textrm{if}\left|w_{r}\right|>\gamma \; $\&$ \; \textrm{sign}(\ensuremath{w_{r}}) = \textrm{sign}(\ensuremath{w_{b}}),\\
w_{b} & \textrm{otherwise}. 
\end{cases}
\label{eq:hyst1}
\end{align}

One could obtain an alternative update rule $\vw_r \leftarrow  (1-\alpha)\vw_r - \alpha \mathbf{\vg}$, as shown in Step 3 of Bop in \autoref{table:Psedo-code}. In this case, the update rule and the corresponding \emph{hysteresis} function are as follows 
\begin{align}
\vw_r & \leftarrow  (1-\alpha)\vw_r - \alpha \mathbf{\vg}, \label{eq:update2}\\
y & =\textrm{hyst2}\left(w_{r},w_b,\gamma\right)\nonumber \\
 & \equiv \begin{cases}
-w_b & \textrm{if}\left|w_{r}\right|>\gamma \; $\&$ \; \textrm{sign}(\ensuremath{w_{r}}) = -\textrm{sign}(\ensuremath{w_{b}}),\\
w_{b} & \textrm{otherwise}.
\end{cases}
\label{eq:hyst}
\end{align}

It could be easily verified that the above two update rules with two different representations of the \emph{hysteresis} function are equivalent to each other: The only difference between (\ref{eq:update1}) and (\ref{eq:update2}) is the sign before the gradient $\mathbf{\vg}$, i.e., the $\vw_r$ in (\ref{eq:update1}) is an exponential moving average \citep{gardner1985exponential} of $\mathbf{\vg}$ while in (\ref{eq:update2}) it is an exponential moving average of $-\mathbf{\vg}$. Such difference is compensated by the difference between (\ref{eq:hyst1}) and (\ref{eq:hyst}). The corresponding curve of $y = \textrm{hyst1}\left(w_{r},w_b,\gamma\right)$
is simply a upside-down flipped version of $y = \textrm{hyst2}\left(w_{r},w_b,\gamma\right)$, which is shown in the rightmost figure in \autoref{fig:binn_loop} (b).

\section{Experimental details}
\label{Sec:Experimental-details}
In this section we list the details for all experiments shown in the main text.

Note that after training BiNNs with BayesBiNN, there are two ways to perform inference during test time:

(1). \textbf{Mean}: One method is to use the predictive mean, where we use Monte Carlo sampling to compute the predictive probabilities for each test sample  $\vx_{j}$  as follows
\begin{align}
\hat{p}_{j,k} \approx  \frac{1}{C}\sum_{c=1}^{C}p\left(y_{j}=k|\vx_{j},\vw^{(c)}\right),
\label{eq:prediction}
\end{align}
where $\vw^{(c)}\sim q(\vw)$ are samples from the Bernoulli distributions with the natural parameters $\vlambda$  obtained by BayesBiNN. 

(2). \textbf{Mode}: The other way is simply to use the mode of the posterior distribution $q(\vw)$, i.e., the sign value of the posterior mean, i.e., $\hat{\vw } = \textrm{sign}(\tanh{\left(\vlambda\right)})$, to make predictions, which  will be denoted as  $C=0$.  

\subsection{Synthetic Data}\label{sec:settings_uncertainty}
\paragraph{Binary Classification}
We used the Two Moons dataset with 100 data points in each class and added Gaussian noise with standard deviation 0.1 to each point. We trained a Multilayer Perceptron (MLP) with two hidden layers of 64 units and tanh activation functions for 3000 epochs, using Cross Entropy as the loss function. Additional train and test settings with respect to the optimizers are detailed in \autoref{tab:settings_moons}. The learning rate $\alpha$ was decayed at fixed epochs by the specified learning rate decay rate. For the STE baseline, we used the Adam optimizer with standard settings.

\begin{table}[h]
    \caption{Train settings for the binary classification experiment using the Two Moons dataset.}
    \label{tab:settings_moons}
    \begin{center}
    \begin{tabular}{@{}lll@{}}
        \toprule
        {\bf Setting} & {\bf BayesBiNN}  & {\bf STE} \\
        \midrule
        Learning rate $\alpha$      & $10^{-3}$    & $10^{-1}$ \\
        Learning rate decay         & 0.1          & 0.1 \\
        Learning rate decay epochs  & [1500, 2500] & [1500, 2500] \\
        Momentum(s) $\beta$         & 0.99         & 0.9, 0.999 \\
        MC train samples $S$        & 5            & - \\
        MC test samples $C$         & 0/10           & - \\
        Temperature $\tau$          & 1            & - \\
        Prior $\vlambda_0$           & $\boldsymbol{0}$ & - \\
        Initialization $\lambda$    & $\pm 15$ randomly & - \\
        \bottomrule
    \end{tabular}
    \end{center}
\end{table}

\paragraph{Regression}
We used the Snelson dataset~\citep{snelson2006sparsegp} with 200 data points to train a regression model. Similar to the Binary Classification experiment, we used an MLP with two hidden layers of 64 units and tanh activation functions, but trained it for 5000 epochs using Mean Squared Error as the loss function. Additionally, we added a batch normalization layer (without learned gain or bias terms) after the last fully connected layer. The learning rate is adjusted after every epoch to slowly anneal from an initial learning rate $\alpha_0$ to a target learning rate $\alpha_T$ at the maximum epoch $T$ using
\begin{equation}\label{eq:decay_formula}
    \alpha_{t+1} = \alpha_t \left( \frac{\alpha_T}{\alpha_0} \right)^{-T}.
\end{equation}
The learning rates and other train and test settings are detailed in~\autoref{tab:settings_snelson}.

\begin{table}[h]
    \caption{Train settings for the regression experiment using the Snelson dataset~\citep{snelson2006sparsegp}.}
    \label{tab:settings_snelson}
    \begin{center}
    \begin{tabular}{@{}lll@{}}
        \toprule
         {\bf Setting} & {\bf BayesBiNN}  & {\bf STE} \\
        \midrule
        Learning rate start $\alpha_0$  & $10^{-4}$   & $10^{-1}$ \\
        Learning rate end $\alpha_T$    & $10^{-5}$   & $10^{-1}$ \\
        Momentum(s) $\beta$             & 0.99        & 0.9, 0.999 \\
        MC train samples $S$            & 1           & - \\
        MC test samples $C$             & 0/10          & - \\
        Temperature $\tau$              & 1           & - \\
        Prior $\boldsymbol{\lambda}_0$               & $\boldsymbol{0}$ & - \\
        Initialization $\boldsymbol{\lambda}$        & $\pm 10$ randomly & - \\
        \bottomrule
    \end{tabular}
    \end{center}
\end{table}

\subsection{MNIST, CIFAR-10 and CIFAR-100}\label{sec:settings_mnist_cifar}
In this section, three well-known image datasets are considered, namely the MNIST, CIFAR-10 and CIFAR-100 datasets. We compare the proposed BayesBiNN with four other popular algorithms, STE Adam, Bop and PMF for BiNNs as well as standard Adam for full-precision weights. For dataset and algorithm specific settings, see \autoref{tab:settings_all_alg}.

\paragraph{MNIST}
All algorithms have been trained using the same MLP detailed in~\autoref{tab:mlp} on mini-batches of size 100, for a maximum of 500 epochs. The loss used was Categorical Cross Entropy. We split the original training data into 90\% train and 10\% validation data and no data augmentation except normalization has been done. We report the best accuracy (averaged over 5 random runs) on the test set corresponding to the highest validation accuracy achieved during training (we do not retrain using the validation set). Note that we tune the hyper-parameters such as learning rate for all the methods including the baselines. The search space for the learning rate is set to be $\left[\text{$10^{-2}$, $3\cdot10^{-3}$, $10^{-3}$, $3\cdot10^{-4}$, $10^{-4}$, $3\cdot10^{-5}$, $10^{-5}$, $10^{-6}$}\right]$ for all methods. 
Moreover, \autoref{tab:Result-differentLR} and \autoref{tab:Result-differentTemperature} shows the results of MNIST with BayesBiNN for different choices of learning rate and temperature.

\begin{table}[h]
    \caption{The MLP architecture used in all MNIST experiments, adapted from~\citep{alizadeh2018empirical}.}
    \label{tab:mlp}
    \begin{center}
    \begin{tabular}{@{}c@{}}
        \toprule
        Dropout (p = 0.2) \\
        Fully Connected Layer (units = 2048, bias = False) \\
        ReLU \\
        Batch Normalization Layer (gain = 1, bias = 0) \\
        \midrule
        Dropout (p = 0.2) \\
        Fully Connected Layer (units = 2048, bias = False) \\
        ReLU \\
        Batch Normalization Layer (gain = 1, bias = 0) \\
        \midrule
        Dropout (p = 0.2) \\
        Fully Connected Layer (units = 2048, bias = False) \\
        ReLU \\
        Batch Normalization Layer (gain = 1, bias = 0) \\
        \midrule
        Dropout (p = 0.2) \\
        Fully Connected Layer (units = 2048, bias = False) \\
        Batch Normalization Layer (gain = 1, bias = 0) \\
        Softmax \\
        \bottomrule
    \end{tabular}
    \end{center}
\end{table}

\begin{table}[h]
    \caption{Test accuracy of MNIST for different initial learning rates. The temperature is $10^{-10}$. Results are averaged over 5 random runs.}
    \label{tab:Result-differentLR}
    \begin{center}
    \begin{tabular}{@{}ccccc@{}}
        \toprule
        Learning rate & $10^{-1}$ & $3\cdot10^{-3}$ & $10^{-3}$ & $3\cdot10^{-4}$  \tabularnewline
 
        \midrule
        Training Accuracy & $99.46 \pm 0.15$ \%  & $99.58 \pm 0.16$ \% & $99.67 \pm 0.09$ \%  & $99.76 \pm 0.09$ \%  \tabularnewline
        
        \midrule
        Validation Accuracy & $98.90 \pm 0.14$ \%  & $98.94 \pm 0.17$ \% & $98.96 \pm 0.13$ \%  & $98.97 \pm 0.12$ \%  \tabularnewline
        
         \midrule
        Test Accuracy & $98.73 \pm 0.11$ \%  & $98.81 \pm 0.07$ \% & $98.83 \pm 0.05$ \%  & $98.84 \pm 0.08$ \%  \tabularnewline
        
        \toprule
        Learning rate & $10^{-4}$ & $3\cdot10^{-5}$ & $10^{-5}$ &   $10^{-6}$ \tabularnewline
        \midrule
        Training Accuracy & $99.85 \pm 0.05$ \%  & $99.83 \pm 0.06$ \% & $99.76 \pm 0.09$ \%  & $99.78 \pm 0.03$ \%  \tabularnewline
        
        \midrule
        Validation Accuracy & $99.02 \pm 0.13$ \%  & $99.02 \pm 0.13$ \% & $99.04 \pm 0.11$ \%  & $99.02 \pm 0.17$ \%  \tabularnewline
         \midrule
        Test Accuracy & $98.86 \pm 0.05$ \%   &  $98.86 \pm 0.05$ \%  & $98.84 \pm 0.08$ \%  & $98.85 \pm 0.05$ \%  \tabularnewline
        
        \bottomrule
    \end{tabular}
    \end{center}
\end{table}

\begin{table}[h]
    \caption{Test accuracy of MNIST for different temperatures. The initial learning rate is $10^{-4}$. Results are averaged over 5 random runs.}
    \label{tab:Result-differentTemperature}
    \begin{center}
    \begin{tabular}{@{}cccccc@{}}
        \toprule
        Temperature & $10^{-3}$ & $10^{-4}$ & $10^{-5}$ & $10^{-6}$ & $10^{-7}$   \tabularnewline
        \midrule
        Training Accuracy & $89.25 \pm 0.22$ \%  & $87.55 \pm 0.50$ \% & $90.22 \pm 0.42$ \%  & $97.37 \pm 0.13$ \%  & $98.27 \pm 0.10$ \%  \tabularnewline
        \midrule
        Validation Accuracy & $90.06 \pm 1.04$ \%  & $90.28 \pm 0.43$ \% & $93.35 \pm 0.48$ \%  & $98.10 \pm 0.17$ \%  & $98.55 \pm 0.16$ \%  \tabularnewline

         \midrule
        Test Accuracy & $90.40 \pm 0.97$ \%  & $90.72 \pm 0.42$ \% & $93.67 \pm 0.50$ \%  & $98.01 \pm 0.05$ \%  & $98.41 \pm 0.10$ \%  \tabularnewline

        \toprule
        Learning rate & $10^{-8}$ & $10^{-9}$ & $10^{-10}$ &   $10^{-11}$ &   $10^{-12}$ \tabularnewline
                
        \midrule
        Training Accuracy & $99.48 \pm 0.08$ \%  & $99.75 \pm 0.14$ \% & $99.85 \pm 0.05$ \%  & $99.81 \pm 0.04$ \%  & $99.82 \pm 0.07$ \%  \tabularnewline
        
        \midrule
        Validation Accuracy & $98.92 \pm 0.13$ \%  & $99.00 \pm 0.13$ \% & $99.02 \pm 0.14$ \%  & $99.02 \pm 0.12$ \%  & $99.02 \pm 0.13$ \%  \tabularnewline
         \midrule
        Test Accuracy &  $98.82 \pm 0.05$ \%  & $98.81 \pm 0.08$ \%  & $98.86 \pm 0.05$ \% & $98.86 \pm 0.06$ \%  & $98.84 \pm 0.04$ \%  \tabularnewline
        
        \bottomrule
    \end{tabular}
    \end{center}
\end{table}

\paragraph{CIFAR-10 and CIFAR-100}
We trained all algorithms on the Convolutional Neural Network (CNN) architecture detailed in~\autoref{tab:cnn} on mini-batches of size 50, for a maximum of 500 epochs. The loss used was Categorical Cross Entropy. We split the original training data into 90\% train and 10\% validation data. For data augmentation during training, the images were normalized, a random 32 $\times$ 32 crop was selected from a 40 $\times$ 40 padded image and finally a random horizontal flip was applied. In the same manner as \citet{osawa2019practical}, we consider such data augmentation as effectively increasing the dataset size by a factor of 10 (4 images for each corner, and one central image, and the horizontal flipping step further doubles the dataset size, which gives a total factor of 10). We report the best accuracy (averaged over 5 random runs) on the test set corresponding to the highest validation accuracy achieved during training. In addition, we tune the hyper-parameters, such as the learning rate, for all the methods including the baselines. The search space for the learning rate is set to be $\left[\text{$10^{-2}$, $3\cdot10^{-3}$, $10^{-3}$, $3\cdot10^{-4}$, $10^{-4}$, $3\cdot10^{-5}$, $10^{-5}$, $10^{-6}$}\right]$ for all methods. 

\begin{table}[h]
    \caption{The CNN architecture used in all CIFAR-10 and CIFAR-100 experiments, inspired by VGG and used in \citet{alizadeh2018empirical}.}
    \label{tab:cnn}
    \begin{center}
    \begin{tabular}{@{}c@{}}
        \toprule
        Convolutional Layer (channels = 128, kernel-size = 3 $\times$ 3, bias = False, padding = same) \\
        ReLU \\
        Batch Normalization Layer (gain = 1, bias = 0) \\
        Convolutional Layer (channels = 128, kernel-size = 3 $\times$ 3, bias = False, padding = same) \\
        ReLU \\
        Max Pooling Layer (size = 2 $\times$ 2, stride = 2 $\times$ 2) \\
        Batch Normalization Layer (gain = 1, bias = 0) \\
        \midrule
        Convolutional Layer (channels = 256, kernel-size = 3 $\times$ 3, bias = False, padding = same) \\
        ReLU \\
        Batch Normalization Layer (gain = 1, bias = 0) \\
        Convolutional Layer (channels = 256, kernel-size = 3 $\times$ 3, bias = False, padding = same) \\
        ReLU \\
        Max Pooling Layer (size = 2 $\times$ 2, stride = 2 $\times$ 2) \\
        Batch Normalization Layer (gain = 1, bias = 0) \\
        \midrule
        Convolutional Layer (channels = 512, kernel-size = 3 $\times$ 3, bias = False, padding = same) \\
        ReLU \\
        Batch Normalization Layer (gain = 1, bias = 0) \\
        Convolutional Layer (channels = 512, kernel-size = 3 $\times$ 3, bias = False, padding = same) \\
        ReLU \\
        Max Pooling Layer (size = 2 $\times$ 2, stride = 2 $\times$ 2) \\
        Batch Normalization Layer (gain = 1, bias = 0) \\
        \midrule
        Fully Connected Layer (units = 1024, bias = False) \\
        ReLU \\
        Batch Normalization Layer (gain = 1, bias = 0) \\
        \midrule
        Fully Connected Layer (units = 1024, bias = False) \\
        ReLU \\
        Batch Normalization Layer (gain = 1, bias = 0) \\
        \midrule
        Fully Connected Layer (units = 1024, bias = False) \\
        Batch Normalization Layer (gain = 1, bias = 0) \\
        Softmax \\
        \bottomrule
    \end{tabular}
    \end{center}
\end{table}

\begin{table}[h]
    \caption{Algorithm specific train settings for MNIST, CIFAR-10, and CIFAR-100.}
    \label{tab:settings_all_alg}
    \begin{center}
    \begin{tabular}{@{}lllll@{}}
        \toprule
        {\bf Algorithm} & {\bf Setting} & {\bf MNIST} & {\bf CIFAR-10} & {\bf CIFAR-100} \\
        \midrule
        \multirow{8}{*}{BayesBiNN}
        & Learning rate start $\alpha_0$    & $10^{-4}$ & $3\cdot10^{-4}$ & $3\cdot10^{-4}$ \\
        & Learning rate end $\alpha_T$      & $10^{-16}$ & $10^{-16}$ & $10^{-16}$ \\
        & Learning rate decay               & Cosine & Cosine & Cosine \\
        & MC train samples $S$              & 1 & 1 & 1 \\
        & MC test samples $C$               & 0 & 0 & 0 \\
        & Temperature $\tau$                & $10^{-10}$ & $10^{-10}$ & $10^{-8}$ \\
        & Prior $\boldsymbol{\lambda}_0$                 & $\boldsymbol{0}$ & $\boldsymbol{0}$ & $\boldsymbol{0}$ \\
        & Initialization $\boldsymbol{\lambda}$          & $\pm 10$ randomly & $\pm 10$ randomly & $\pm 10$ randomly \\
        \midrule
        \multirow{5}{*}{STE Adam}
        & Learning rate start $\alpha_0$    & $10^{-2}$ & $10^{-2}$ & $10^{-2}$ \\
        & Learning rate end $\alpha_T$      & $10^{-16}$ & $10^{-16}$ & $10^{-16}$ \\
        & Learning rate decay               & Cosine & Cosine & Cosine \\
        & Gradient clipping                 & Yes & Yes & Yes \\
        & Weights clipping                  & Yes & Yes & Yes \\
        \midrule
        \multirow{5}{*}{Bop}
        & Threshold $\tau$                  & $10^{-8}$             & $10^{-8}$             & $10^{-9}$ \\
        & Adaptivity rate $\gamma$          & $10^{-5}$             & $10^{-4}$             & $10^{-4}$ \\
        & $\gamma$-decay type               & Step                  & Step                  & Step \\
        & $\gamma$-decay rate               & $10^{\frac{-3}{500}}$ & 0.1                   & 0.1 \\
        & $\gamma$-decay interval (epochs)  & 1                     & 100                   & 100  \\
        \midrule
        \multirow{7}{*}{PMF}
        & Learning rate start               & $10^{-3}$ & $10^{-2}$ & $10^{-2}$ \\
        & Learning rate decay type          & Step & Step & Step \\
        & LR decay interval (iterations)    & 7k & 30k & 30k \\
        & LR-scale                          & 0.2 & 0.2 & 0.2 \\
        & Optimizer                         & Adam & Adam & Adam \\
        & Weight decay                      & 0 & $10^{-4}$ & $10^{-4}$ \\
        & $\rho$                            & 1.2 & 1.05 & 1.05 \\
        \midrule
        \multirow{3}{*}{Adam (Full-precision)}
        & Learning rate start $\alpha_0$    & $3\cdot10^{-4}$ & $10^{-2}$ & $3\cdot10^{-3}$ \\
        & Learning rate end $\alpha_T$      & $10^{-16}$ & $10^{-16}$ & $10^{-16}$ \\
        & Learning rate decay               & Cosine & Cosine & Cosine \\
        \bottomrule
    \end{tabular}
    \end{center}
\end{table}

\begin{table*}[t]
\caption{Detailed results of different optimizers trained on MNIST, CIFAR-10 and CIFAR-100 (Averaged over 5 runs).}
\label{tab:detailed_results}
\begin{center}
\begin{tabular}{@{}lllll@{}}
\toprule
{\bf Dataset} & {\bf Optimizer} & {\bf Train Accuracy}  & {\bf Validation Accuracy}  & {\bf Test Accuracy} \\
\midrule
\multirow{6}{*}{MNIST}
& STE Adam        &  $99.78  \pm 0.10$ \%   &  $99.02 \pm 0.11$ \%   &  $\bold{98.85} \pm 0.09$ \%  \\
& Bop             &  $ 99.23 \pm 0.04$ \%   &  $98.55 \pm 0.05$ \%   &  $98.47 \pm 0.02$ \%  \\
& PMF             &       &  $99.06 \pm 0.01$ \%   &  $98.80 \pm 0.06$ \%  \\
& \textbf{BayesBiNN} (mode)  &  $ 99.85 \pm 0.05$ \%   &  $99.02 \pm 0.13$ \%   &  $\bold{98.86} \pm 0.05$ \%  \\

& \textbf{BayesBiNN} (mean)  &  $ 99.85 \pm 0.05$ \%   &  $99.02 \pm 0.13$ \%   &  $\bold{98.86} \pm 0.05$ \%  \\

\cmidrule(l){2-5}
& Full-precision  &  $ 99.96 \pm 0.02$ \%   &  $99.15 \pm 0.14$ \%   &  $99.01 \pm 0.06$ \% \\
\midrule
\multirow{6}{*}{CIFAR-10}
& STE Adam        & $99.99 \pm 0.01$ \% & $94.25 \pm 0.42$ \% & $\bold{93.55} \pm 0.15$ \% \\

& Bop             & $99.79 \pm 0.03$ \% & $93.49 \pm 0.17$ \% & $93.00 \pm 0.11$ \%\\
& PMF             &    &  $91.87 \pm 0.10$ \%   &  $91.43 \pm 0.14$ \%  \\
& \textbf{BayesBiNN} (mode)  &  $99.96 \pm0.01$ \%   &  $94.23 \pm0.41$ \%       &  $\bold{93.72} \pm0.16$ \%  \\

 & \textbf{BayesBiNN} (mean)  &  $99.96 \pm0.01$ \%   &  $94.23 \pm0.41$ \%       &  $\bold{93.72} \pm0.15$ \%  \\

\cmidrule(l){2-5}
& Full-precision   & $100.00 \pm0.00$ \%  & $94.54 \pm0.29$ \%  &$93.90 \pm0.17$ \% \\
\midrule
\multirow{6}{*}{CIFAR-100}
& STE Adam        & $99.06  \pm 0.15$ \% & $74.09 \pm 0.15$ \% & $72.89 \pm 0.21$ \%  \\
& Bop             & $90.09 \pm 0.57$ \% & $69.97 \pm 0.29$ \% & $69.58 \pm 0.15$ \% \\
& PMF             &     &  $69.86 \pm 0.08$ \%   &  $70.45 \pm 0.25$ \%  \\
& \textbf{BayesBiNN} (mode)   & $98.02 \pm 0.18$ \%  & $74.76 \pm0.41$ \%  &$\bold{73.68} \pm0.31$ \%  \\

 & \textbf{BayesBiNN} (mean)   & $98.02 \pm 0.18$ \%  & $74.76 \pm0.41$ \%  &$\bold{73.65} \pm0.41$ \%  \\

\cmidrule(l){2-5}
& Full-precision   & $99.89 \pm0.02$ \%  & $75.89 \pm0.41$ \%  &$74.83 \pm0.26$ \% \\
\bottomrule
\end{tabular}
\end{center}
\end{table*}

\begin{figure*}[h!!]
\begin{center}
\includegraphics[width = 1\textwidth]{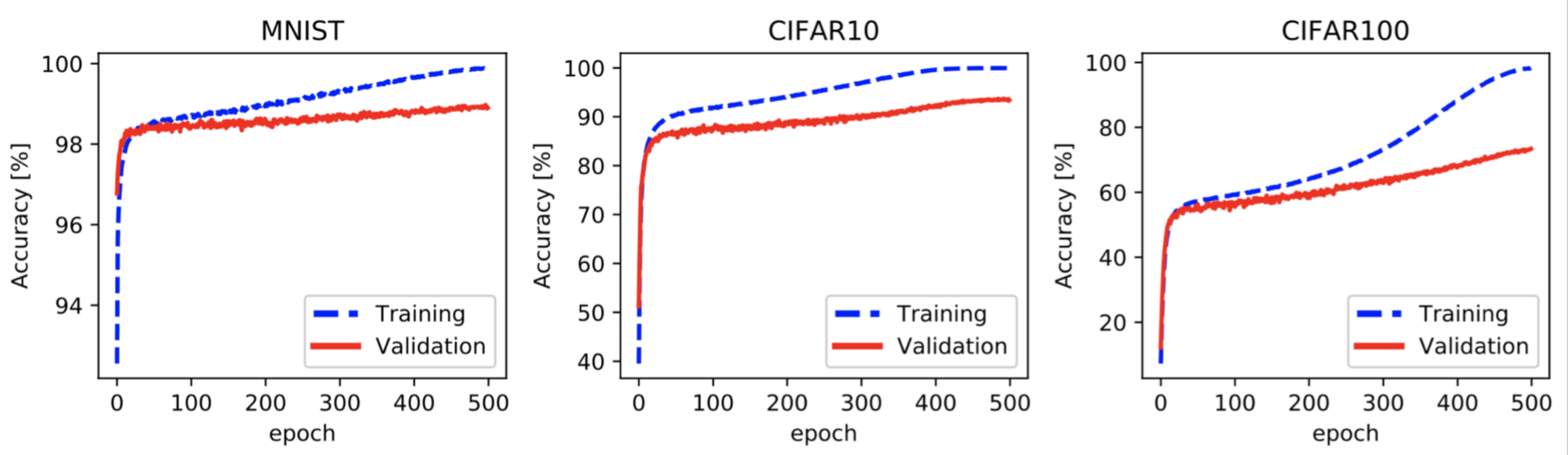}
\end{center}
\caption{Training/Validation accuracy for MNIST, CIFAR-10 and CIFAR100 with BayesBiNN optimizer (Averaged over 5 runs).}
\label{Figure_train_plots}
\end{figure*}

\subsection{Comparison with LR-net}\label{sec:LR-net}
We also compare the proposed BayesBiNN with the LR-net method in \citet{shayer2017learning} for MNIST and CIFAR-10. As the code for the LR-net is not open-source, we performed experiments with BayesBiNN following the same experimental settings in \citet{shayer2017learning} and then compared the results with the reported results in their paper. In specific, the network architectures for MNIST and CIFAR-10 are the same as \citet{shayer2017learning}, except that we added BN after the FC layers. However, we kept all layers binary and did not learn the BN parameters, nor did we use dropout as in \citet{shayer2017learning}. The dataset pre-processing follows the same settings in \citet{shayer2017learning} and is  similar to that described in \autoref{sec:image-classification}, except that there is no split of the training set into training and validation sets. As a result, as in \citet{shayer2017learning}, we report the test accuracies after 190 epochs and 290 epochs for MNIST and CIFAR-10, respectively. Note that the hyper-parameter settings of BayesBiNN are the same as those in \autoref{tab:settings_all_alg} for MNIST and CIFAR-10. 
The results are shown in ~\autoref{tab:LR-net}. The proposed BayesBiNN achieves similar performance (slightly better for CIFAR-10) to the LR-net. Note that the LR-net method used pre-trained models to initialize the weights of BiNNs, while BayesBiNN trained BiNNs from scratch without using pre-trained models.

\begin{table}[h]
    \caption{Test accuracy of BayesBiNN and LR-net trained on MNIST, CIFAR-10. Results for BayesBiNN are averaged over 5 random runs.}
    \label{tab:LR-net}
    \begin{center}
    \begin{tabular}{@{}lll@{}}
        \toprule
        {\bf Optimizer} & {\bf MNIST}  & {\bf CIFAR-10}  \\
        \midrule
        LR-net \citet{shayer2017learning}   &  $99.47$ \%      &  $93.18$\%  \\
        \textbf{BayesBiNN} (mode)  &  ${99.50} \pm 0.02$ \%    &  ${93.97} \pm0.11$ \% \\
        \bottomrule
    \end{tabular}
    \end{center}
\end{table}

\subsection{Continual learning with binary neural networks}\label{sec:settings_continual_learning}
For the continual learning experiment, we used a three-layer MLP, detailed in \autoref{tab:mlp_CL}, and trained it using the Categorical Cross Entropy loss. Specific training parameters are given in \autoref{tab:settings_CL}. There is no split of the original MNIST training data in the continual learning case. No data augmentation except normalization has been performed. 

\begin{table}[h]
    \caption{The MLP architecture used for continual learning \citep{nguyen2017variational}}
    \label{tab:mlp_CL}
    \begin{center}
    \begin{tabular}{@{}c@{}}
        \toprule
        Fully Connected Layer (units = 100, bias = False) \\
        ReLU \\
        Batch Normalization Layer (gain = 1, bias = 0) \\
        \midrule
        Fully Connected Layer (units = 100, bias = False) \\
        ReLU \\
        Batch Normalization Layer (gain = 1, bias = 0) \\
        \midrule
        Fully Connected Layer (units = 100, bias = False) \\
        ReLU \\
        Batch Normalization Layer (gain = 1, bias = 0) \\
        Softmax \\
        \bottomrule
    \end{tabular}
    \end{center}
\end{table}

\begin{table}[h]
    \caption{Algorithm specific train settings for continual learning on permuted MNIST.}
    \label{tab:settings_CL}
    \begin{center}
    \begin{tabular}{@{}lllll@{}}
        \toprule
        {\bf Algorithm} & {\bf Setting} & {\bf Permuted MNIST} \\
        \midrule
        \multirow{9}{*}{BayesBiNN}
        & Learning rate start $\alpha_0$    & $10^{-3}$ \\
        & Learning rate end $\alpha_T$      & $10^{-16}$ \\
        & Learning rate decay               & Cosine \\
        & MC train samples $S$              & 1 \\
        & MC test samples $C$               & 100\\
        & Temperature $\tau$                & $10^{-2}$ \\
        & Prior $\vlambda_0$                 & learned $\vlambda$ of the previous task\\
        & Initialization $\vlambda$          & $\pm 10$ randomly \\
        & Batch size                        & 100 \\
        & Number of epochs                  & 100\\
        \bottomrule
    \end{tabular}
    \end{center}
\end{table}

\section{Author Contributions Statement}
M.E.K. conceived the idea of training Binary neural networks using the Bayesian learning rule. X.M. derived the BayesBiNN algorithm, studied its connections to STE and Bop, and wrote the first proof-of-concept experiments. R.B. fixed a few issue with the original implementation and re-organized the PyTorch code. R.B. also designed and performed the experiments on synthetic data presented in Section 4.1. X.M. did most of the experiments with some help from R.B. All the authors were involved in writing, revising and proof-reading the paper. 
\end{document}